\documentclass{article}
\usepackage{ijcai22}

\usepackage{times}
\usepackage{soul}
\usepackage{url}
\usepackage[hidelinks]{hyperref}
\usepackage[utf8]{inputenc}
\usepackage[small]{caption}
\usepackage{graphicx}
\usepackage{amsmath}\allowdisplaybreaks
\usepackage{amsthm}
\usepackage{booktabs}
\usepackage{algorithm}
\usepackage{algorithmic}
\usepackage{subfigure}
\usepackage{bm}
\usepackage{mathtools}
\usepackage{amssymb}
\mathtoolsset{showonlyrefs,showmanualtags}

\newtheorem{theorem}{Theorem}
\newtheorem{definition}{Definition}
\newtheorem{lemma}{Lemma}

\newcommand{\bmx}{\bm{x}}
\newcommand{\bmy}{\bm{y}}
\newcommand{\opgamma}{\mathop{\gamma}\nolimits}
\newcommand\expect[1]{\mathbb{E}[#1]}
\newcommand\citeny[1]{\citeauthor{#1}~\shortcite{#1}}
\newcommand{\opt}{\mathrm{OPT}}
\newcommand{\bmth}{\bm{\theta}}
\newcommand{\prob}{\mathrm{P}}

\title{Robust Subset Selection by Greedy and Evolutionary Pareto Optimization}

\author{
    Anonymous
}

\author{
	Chao Bian$\!$\and
	Yawen Zhou
	\And
	Chao Qian\thanks{This work was supported by the NSFC (62022039) and the	CAAI-Huawei MindSpore Open Fund. Chao Qian is the corresponding author. The conference version of this paper has appeared at IJCAI’22. }
	\affiliations
	State Key Laboratory for Novel Software Technology, Nanjing University, Nanjing 210023, China
	\emails
    bianc@lamda.nju.edu.cn, 
	zyw8769@gmail.com,
	qianc@lamda.nju.edu.cn
}

\begin{document}

\maketitle

\begin{abstract}
Subset selection, which aims to select a subset from a ground set to maximize some objective function, arises in various applications such as influence maximization and sensor placement. In real-world scenarios, however, one often needs to find a subset which is robust against (i.e., is good over) a number of possible objective functions due to uncertainty, resulting in the problem of robust subset selection. This paper considers robust subset selection with monotone objective functions, relaxing the submodular property required by previous studies. We first show that the greedy algorithm can obtain an approximation ratio of $1-e^{-\beta\opgamma}$, where $\beta$ and $\opgamma$ are the correlation and submodularity ratios of the objective functions, respectively; and then propose EPORSS, an evolutionary Pareto optimization algorithm that can utilize more time to find better subsets. We prove that EPORSS can also be theoretically grounded, achieving a similar approximation guarantee to the greedy algorithm. In addition, we derive the lower bound of $\beta$ for the application of robust influence maximization, and further conduct experiments to validate the performance of the greedy algorithm and EPORSS. 
\end{abstract}

\section{Introduction}

The subset selection problem is to select a subset of size at most $k$ from a total set of $n$ items for optimizing some objective function $f$, which arises in many applications, such as influence maximization~\cite{kempe2003maximizing}, sensor placement~\cite{krause2008near}, document summarization~\cite{lin2011class}, and unsupervised feature selection~\cite{feng-aaai19}, etc.
This problem is NP-hard in general, 
and much effort has been devoted to the design of polynomial-time approximation algorithms.

The greedy algorithm, which iteratively selects one item with the largest marginal gain on $f$, is a very popular algorithm for solving the subset selection problem.
For a monotone submodular objective function $f$, it can achieve the optimal polynomial-time approximation guarantee of $1-1/e$~\cite{nemhauser1978best}.
Due to its high-efficiency and theoretically-grounded performance, it is most favored and has been studied extensively~\cite{minoux1978accelerated,mirzasoleiman-aaai15-lazier,hassidim2017greedy,qian2018approximation,friedrich2019greedy}.


In real-world applications of subset selection, however, we may want to find a subset which 
is robust against a number of possible objective functions, i.e., the subset is good on the worst of these objective functions.
For example, in influence maximization, the influence of a set of social network users has significant uncertainty due to different models and parameters, and the goal is to optimize a set of influence functions simultaneously, in which one function is assured to describe the influence process exactly (but which one is not told)~\cite{he-kdd16-robust-inf}. 
In sensor placement, sensors are often used to monitor various phenomena such as temperature and humidity at the same time, whose observations are modeled by different functions, and the goal is to find a subset of candidate locations to place sensors such that the sensors can perform well on all the objective functions~\cite{anari-arxiv18-off-online}.

To address the above more complicated real-world scenario of subset selection,
\citeny{krause-jmlr08-rsos} considered the robust subset selection problem
\begin{align}\label{eq:robust-subsetsel-intro}
	\mathop{\arg\max}\nolimits_{X\subseteq V} \min\nolimits_{1\le i\le m}f_i(X) \quad \text{s.t.}\quad |X|\leq k,
\end{align}
where $f_1,f_2,\ldots,f_m$ are monotone submodular objective functions and $|\cdot|$ denotes the size of a set.
They proved that there cannot exist any polynomial time approximation algorithm unless P=NP, and proposed an algorithm SATURATE, which can find a subset matching the optimal worst-case objective value but with size larger than $k$. 
\citeny{anari-arxiv18-off-online} extended the cardinality constraint $|X| \leq k$ in Eq.~\eqref{eq:robust-subsetsel-intro} to a matroid constraint, and provided a bi-criteria algorithm which can find $O(\log (m/\epsilon))$ feasible subsets whose union achieves an approximation ratio of $1-\epsilon$, where $\epsilon \in (0,1)$.
\citeny{udwani-nips18-multi}   designed a fast and practical algorithm for the case of $m=o(k)$.
\citeny{he-kdd16-robust-inf} modified SATURATE for the application of robust influence maximization, and showed that a $(1-1/e)$-approximation ratio can be achieved when enough extra seed users may be selected. 
\citeny{iyer2019unify} considered both robust submodular minimization and maximization problems, and proposed a general framework to transform the algorithm for one problem to the other.
Recently, \citeny{houtac21} introduced the notion of correlation ratio and proposed a modified greedy algorithm which can achieve an approximation guarantee with respect to the correlation ratio. 

Note that all the above-mentioned previous studies considered monotone submodular objective functions. However, the objective functions of some applications such as influence maximization under general cascade model~\cite{kempe2003maximizing} and sparse regression~\cite{das2011submodular} can be non-submodular. 
This paper thus considers the problem of robust subset selection with monotone objective functions, that is, $f_1,f_2,\ldots,f_m$ are monotone but not necessarily submodular. 
First, we show that the classic greedy algorithm which iteratively adds one item maximizing the gain on the worst-case objective function can achieve an approximation ratio of $1-e^{-\beta\opgamma}$, where $\beta$ and $\opgamma$ are the correlation ratio (slightly different from that proposed in~\cite{houtac21}) and the submodularity ratio of the objective functions, respectively. 

Though the greedy algorithm is efficient and theoretically grounded, its performance may be limited in practice due to the greedy nature. Next, we propose an evolutionary Pareto optimization algorithm, EPORSS, which is an anytime algorithm that can use more time to find better subsets. EPORSS first reformulates the original robust subset selection problem as a bi-objective optimization problem that maximizes $\min_{1\le i\le m}f_i$ and minimizes the subset size simultaneously, then employs a multi-objective evolutionary algorithm to solve it, and finally selects the best feasible solution (i.e., subset) from the generated population. We prove that EPORSS can achieve an  approximation ratio similar to that of the greedy algorithm.

The submodularity ratio $\gamma$ has been analyzed before (e.g., \cite{bian2017guarantees,elenberg2016restricted,qian2019maximizing}), and here we prove a lower bound of the correlation ratio $\beta$ for the application of robust influence maximization, to show the applicability of the derived approximation guarantees of the greedy algorithm and EPORSS. 
Experimental results on robust influence maximization show that the greedy algorithm can achieve competitive performance to the well-known algorithm SATURATE, but requires less time; EPORSS using more time can achieve superior performance over the greedy algorithm and SATURATE. 
Therefore, these two algorithms can be used under the limited and abundant computation budgets, respectively, helping solve diverse practical problems of robust subset selection better.

\section{Preliminaries} 

Given a finite nonempty set $V=\{v_1,v_2,\ldots,v_n\}$, we study the functions $f:2^V \rightarrow \mathbb{R}$ defined on subsets of $V$. A set function $f$ is monotone if for any $X \subseteq Y$, $f(X) \leq f(Y)$. In this paper, we consider monotone functions and assume that they are normalized, i.e., $f(\emptyset)=0$. A set function $f$ is submodular~\cite{nemhauser1978analysis} if
for any $X \subseteq Y \subseteq V$, 
\begin{equation}
		f(Y)-f(X) \leq \sum\nolimits_{v \in Y \setminus X} \big(f(X \cup\{ v\})-f(X)\big). 
\end{equation}
The submodularity ratio in Definition~\ref{def:subratio} characterizes how close a set function $f$ is to submodularity. For a monotone function $f$, it holds that $0\le \opgamma_{X,b}(f) \le 1$, and $f$ is submodular iff $\opgamma_{X,b}(f)\! =\! 1$ for any $X$ and $b$. For some concrete non-submodular applications, lower bounds on $\opgamma_{X,b}(f)$ have been derived~\cite{bian2017guarantees,elenberg2016restricted,qian2019maximizing}. When $f$ is clear, we will omit $f$ and use $\opgamma_{X,b}$ for short.
\begin{definition}[Submodularity Ratio~\cite{das2011submodular}]\label{def:subratio}
	Let $f$ be a non-negative set function. The submodularity ratio of $f$ with respect to a set $X$ and a parameter $b\geq 1$ is
	$$
	\opgamma_{X,b}(f)\!=\!\min_{L \subseteq X, S: |S|\leq b, S \cap L =\emptyset} \frac{\sum_{v \in S} \big(f(L \cup \{v\})\!-\!f(L)\big)}{f(L \cup S)-f(L)}.
	$$
\end{definition}
The subset selection problem as presented in Definition~\ref{def:subset} is to select a subset $X$ of $V$ such that a given monotone objective function $f$ is maximized with the constraint $|X|\leq k$, where $|\cdot|$ denotes the size of a set. 
\begin{definition}[Subset Selection]\label{def:subset} Given a set of  items $V=\{v_1,v_2,\ldots,v_n\}$, a monotone objective function $f$, and a budget $k$, to find\vspace{-0.3em}
	\begin{equation}\label{eq:subsetsel}
		\mathop{\arg\max}\nolimits_{X\subseteq V} f(X) \quad \text{s.t.}\quad |X|\leq k.
		\vspace{-0.3em}
	\end{equation}
\end{definition}
It has many applications such as influence maximization~\cite{kempe2003maximizing} and sensor placement~\cite{krause2008near}. Influence maximization is to select a subset of users from a social network to maximize its influence spread, which often appears in viral marketing. Sensor placement is to select a few places to install  sensors such that the uncertainty is mostly reduced, which often appears in spatial monitoring of certain phenomena.

The original subset selection problem in Definition~\ref{def:subset} considers optimizing only one objective function. However, in real-world scenarios, one may need to find a subset which performs well on a number of possible objective functions. 
For example, in viral marketing, the company may want to identify a set of users to promote several products (e.g., shirts and sweaters) simultaneously, while the influences of users on different products can be different~\cite{kalimeris-imcl19-robust-inf}; even for one product, perturbations or uncertainty regarding the parameters of diffusion models will result in different objective functions. 
In spatial monitoring of certain phenomena, sensors are often used to measure various parameters such as temperature and humidity at the same time, giving rise to different objective functions~\cite{anari-arxiv18-off-online}; besides, the data obtained in real world is often noisy, which will lead to perturbations on the objective functions. 
In both above-mentioned applications, we need to 
find a subset that is robust against all the possible objective functions,
i.e., optimize the worst of these objective functions.
More evidences can be seen on the applications of robust column subset selection~\cite{zhu-aaai15-column} and robust sparse regression~\cite{thompson2022}.
In this paper, we focus on the general robust subset selection problem, presented as follows.
\begin{definition}[Robust Subset Selection]\label{def:robust-subsetsel} Given a set of items $V=\{v_1,v_2,\ldots,v_n\}$,  $m$ monotone objective functions $f_1,f_2,\ldots,f_m$, and a budget $k$, to find\vspace{-0.3em}
	\begin{equation}\label{eq:robust-subsetsel}
		\mathop{\arg\max}\nolimits_{X\subseteq V} F(X) \quad \text{s.t.}\quad |X|\leq k,
		\vspace{-0.3em}
	\end{equation}
	where $F(X)=\min_{1\le i\le m}f_i(X)$ denotes the worst-case objective value of $X$. 
\end{definition}
This problem is very hard even when the objective functions $f_1,f_2,\ldots,f_m$ satisfy the submodular property, because it has been proved that polynomial-time algorithms cannot achieve any constant approximation ratio. By introducing a quantity $\beta^*$, which reflects the correlation of the objective functions,  \citeny{houtac21} proposed the modified greedy algorithm with an approximation ratio of $1-e^{-\beta^*}$.
If allowing the subset size larger than $k$,  \citeny{krause-jmlr08-rsos} showed that the SATURATE algorithm can achieve the optimal $F$ value.

Note that the correlation of the objective functions relies on the behavior of the modified greedy algorithm~\cite{houtac21}. Here we generalize this notion in Definition 4, which will be used in our theoretical analysis.
Intuitively, the correlation ratio reflects the similarity of the objective functions to be optimized. A large correlation ratio implies that all the objective functions can be improved substantially by adding a common item into the subset $X$. For monotone objective functions, $0\le \beta_X\le 1$, and $\beta_X=1$ means that there exists one common item, adding which makes all the functions achieve the largest improvement.
\begin{definition}[Correlation Ratio]\label{def:correlation-ratio}
	The correlation ratio of functions $f_1,f_2,\ldots,f_m$ with respect to a set $X$ is 
	\begin{equation}\label{eq:correlation-X}
		\beta_{X}=\max_{v\in V\backslash X}\min_{1\le i\le m} \frac{f_i(X\cup \{v\})-f_i(X)}{f_i(X\cup \{v^i\})-f_i(X)},
	\end{equation}
	where $v^i\in \arg\max_{v\in V\backslash X} f_i(X\cup \{v\})$.
\end{definition}

\section{The Greedy Algorithm}

To solve the robust subset selection problem, a nature way is to apply the greedy process. 
As shown in Algorithm~\ref{alg:Greedy}, the greedy algorithm iteratively adds one item with the largest improvement on $F$ until $k$ items are selected.
Let $\opt=\max_{X\subseteq V,|X|\le k}F(X)$ denote the optimal function value of the problem in Definition~\ref{def:robust-subsetsel}, and $X_j$ ($0\le j\le k$) denote the subset generated in the $j$-th iteration (note that $X_0=\emptyset$) of Algorithm~\ref{alg:Greedy}.
We prove in Theorem~\ref{thm:greedy} that the greedy algorithm can achieve an approximation ratio of $1-e^{-\beta\opgamma}$, where $\beta=\min_{0\le j\le k-1}\beta_{X_j}$ denotes the minimum correlation ratio with respect to the first $k$ subsets generated by the greedy algorithm, and $\opgamma=\min_{1\le i\le m} \opgamma_{X_{k-1},k}(f_i)$ denotes the minimum submodularity ratio of all the objective functions with respect to $X_{k-1}$ and $k$.
\begin{theorem}\label{thm:greedy}
	For robust subset selection in Definition~\ref{def:robust-subsetsel},  the greedy algorithm finds a subset $X\subseteq V$ with $|X|=k$ and
	\begin{equation}\label{eq:greedy-approx}
		F(X) \geq (1-e^{-\beta\opgamma}) \cdot \opt.
	\end{equation}
\end{theorem}
When the number $m$ of functions is 1, i.e., the robust subset selection problem in Definition~\ref{def:robust-subsetsel} degrades to the original subset selection problem in Definition~\ref{def:subset}, we have $\beta=1$, and the approximation ratio becomes $1-e^{-\gamma}$, which is the optimal polynomial-time approximation guarantee~\cite{harshaw2019submodular}.
When $f_1,f_2,\ldots,f_m$ are submodular, we have $\gamma=1$, and the approximation ratio becomes $1-e^{-\beta}$, which is similar to that achieved by the modified greedy algorithm~\cite{houtac21}.  
Though the two approximation guarantees  
cannot be compared exactly due to the slightly different definitions of the correlation ratio, our greedy algorithm is faster than the modified greedy algorithm~\cite{houtac21}.
The greedy algorithm in Algorithm~\ref{alg:Greedy} performs $(n-k/2+1/2)k$ number of worst-case objective function evaluations in total, while the modified greedy algorithm needs to first compute the largest improvement on each objective function in each iteration, thus requires extra $(n\!-\!k/2+1/2)k$ number of  worst-case objective function evaluations in total.
The detailed procedure of the modified greedy algorithm is provided in the appendix due to space limitation. 

The proof of Theorem~\ref{thm:greedy} relies on Lemma~\ref{lem:onestep-F}, which intuitively means that for any subset, the inclusion of one specific item can improve $F$ by at least a quantity proportional to the current distance to the optimum. 
Lemma~\ref{lem:onestep-f_i} shows a similar result with respect to a single objective function $f_i$, which will be used in the proof of Lemma~\ref{lem:onestep-F}.
The detailed proof of  Lemma~\ref{lem:onestep-F} is provided in the appendix due to space limitation. 
\begin{lemma}\cite{qian2016parallel}\label{lem:onestep-f_i}
	For any $1\le i\le m$ and $X \subseteq V$, there exists one item $v \in V \setminus X$ such that
	\begin{align}\label{eq:inclusion-gamma}
		f_i(X \cup \{v\})-f_i(X) \geq (\opgamma_{X,k}(f_i)/k) \cdot (\opt_i-f(X)).
	\end{align}
	where $\opt_i=\max_{X\subseteq V,|X|\le k}f_i(X)$ denotes the optimal function value of the problem in Definition~\ref{def:subset} with respect to the objective function $f_i$. 
\end{lemma}
\begin{lemma}\label{lem:onestep-F}
	For any $X \subseteq V$, there exists one item $v \in V \setminus X$ such that
	\begin{equation}\label{eq:inclusion-gamma-F}
		F(X \cup \{v\})-F(X) \geq (\beta_X \opgamma_{X,k}^{\min}/k)\cdot  (\opt-F(X)),
	\end{equation}
	where $\opgamma_{X,k}^{\min}=\min_{1\le i\le m} \opgamma_{X,k}(f_i)$ denotes the minimum submodularity ratio of all the objective functions with respect to $X$ and $k$.
\end{lemma}

\begin{algorithm}[t]
	\caption{Greedy Algorithm}
	\label{alg:Greedy}
	\textbf{Input}: all items $V=\{v_1,v_2,\ldots,v_n\}$, the objective function $F=\min_{1\le i\le m}f_i$,
	and a budget $k$\\
	\textbf{Output}: a subset of $V$ with $k$ items\\
	\textbf{Process}:
	\begin{algorithmic}[1]
		\STATE Let $j=0$ and $X_j=\emptyset$;
		\WHILE {$j<k$}
		\STATE  Let $v^*=\arg\max_{v \in V \setminus X_j} F(X_j \cup \{v\})$;
		\STATE  Let $X_{j+1}=X_{j} \cup \{v^*\}$, and $j=j+1$
		\ENDWHILE
		\RETURN {$X_k$}
	\end{algorithmic}
\end{algorithm}

\begin{proof}[Proof of Theorem~\ref{thm:greedy}] 
	We will show that for any $0\le j\le k$,
	\begin{equation}\label{eq:greedy-induction}
		\opt-F(X_j)\le (1-\beta\opgamma/k)^j \cdot \opt.
	\end{equation}
	For $j=0$, Eq.~\eqref{eq:greedy-induction} holds because $X_0 = \emptyset$ and $\forall i: f_i(\emptyset)=0$.
	Suppose that Eq.~\eqref{eq:greedy-induction} holds for some $j$ (where $0 \leq j \leq k-1$), we will prove that it also holds for $j+1$.
	Let $v^*\in \mathop{\arg\max}_{v\in V\backslash X_j}F(X_j\cup\{v\})$. 
	By Lemma~\ref{lem:onestep-F}, we have 
	\begin{align}
		F(X_j\!\cup\!\{v^*\})\!-\!F(X_j) 
		&\!\ge\! (\beta_{X_j}\!\!\opgamma_{X_j,k}^{\min}/k)\!\cdot\! (\opt\!-\!F(X_j))\\
		&\!\ge\! (\beta\opgamma/k)\cdot (\opt-F(X_j)),
	\end{align}
	where the second inequality is by $\beta=\min_{0\le j\le k-1}\beta_{X_j} \leq \beta_{X_j}$, and $\opgamma\!=\!\min_{1\le i\le m} \opgamma_{X_{k-1},k}(f_i) \leq \min_{1\le i\le m} \opgamma_{X_{j},k}(f_i)$ $= \opgamma_{X_j,k}^{\min}{k}$ due to the fact that 
	$\opgamma_{X,b}(f)\le \opgamma_{Y,b}(f)$ for any $Y\subseteq X$.
	According to lines~3-4 in Algorithm~\ref{alg:Greedy}, $X_{j+1}$ is generated by combining $X_j$ with $v^*$. Thus, we have
	\begin{align}
		&\opt-F(X_{j+1})
		= \opt-F(X_j\cup \{v^*\})\\
		&\le (1-\beta\opgamma/k)(\opt\!-\!F(X_j))\le (1-\beta\opgamma/k)^{j+1}\opt,
	\end{align}
	where 
	the last inequality holds by the induction hypothesis.
	
	Thus, we have shown that Eq.~\eqref{eq:greedy-induction} holds for any $0\le j\le k$.
	By setting $j=k$ and applying $1+x\le e^x$ for any $x\in\mathbb{R}$, the theorem holds.
\end{proof}

\section{ The EPORSS Algorithm}

Though the greedy algorithm is efficient and its approximation performance can be theoretically bounded, its performance may be limited in practice due to the greedy nature. Thus, we further propose an Evolutionary Pareto Optimization~\cite{friedrich2015maximizing,qian.nips15}  algorithm for Robust Subset Selection, called EPORSS, which can use more time to find better solutions. 

Let a Boolean vector $\bm{x} \in \{0,1\}^n$ represent a subset $X$ of $V$, where the $i$-th bit $x_i=1$ if $v_i \!\in\! X$ and $x_i=0$ otherwise.  For convenience, we will not distinguish $\bm{x}$ and its corresponding subset. 
EPORSS reformulates the original problem in Definition~\ref{def:robust-subsetsel} as a bi-objective maximization problem:
\vspace{-1.5em}
\begin{gather}
	\arg\max\nolimits_{\bm{x} \in \{0,1\}^n} \;\;  (g_1(\bm{x}),\;g_2(\bm{x})), \\
			\text{where }
	g_1(\bm{x}) = \begin{cases}
		-\infty, & |\bm x|\geq 2k\\
		F(\bm x), &\text{otherwise}
	\end{cases},\quad
	g_2(\bm{x}) = -|\bm x|.
\end{gather}
That is, EPORSS maximizes the original objective $F$ and minimizes the subset size $|\bmx|$ simultaneously. Note that infeasible solutions can cause “shortcuts” and be useful in the evolutionary search, thus $g_1$ only punishes the solutions with size at least $2k$ so that 
the infeasible solutions with size belonging to $\{k+1,\ldots,2k-1\}$ can participate in the optimization process.
Meanwhile,
too many infeasible solutions may lead to longer runtime, thus  the setting of the threshold $2k$ is to make a tradeoff between performance and runtime.

In the bi-objective setting, the domination relationship in Definition~\ref{def-dom} is used to compare two solutions.  Note that  $\bmx$ and $\bmy$ are \emph{incomparable} if neither $\bm{x} \succeq \bmy$ nor $\bmy\succeq \bm{x}$.
\begin{definition}[Domination]\label{def-dom}For two solutions $\bm{x}$ and $\bm{y}$,
	\begin{itemize}
		\item $\bm{x}$ \emph{weakly dominates} $\bm{y}$ (denoted as $\bm{x} \succeq \bm{y}$) if $g_1(\bm{x}) \geq g_1(\bm{y}) \wedge g_2(\bm{x}) \geq g_2(\bm{y})$;
		\item $\bm{x}$ \emph{dominates} $\bm{y}$ (denoted as $\bm{x} \succ \bm{y}$) if ${\bm{x}} \succeq \bm{y}$ and $g_1({\bm{x}}) >g_1(\bm{y}) \vee g_2(\bm{x}) > g_2(\bm{y})$.
	\end{itemize}
\end{definition}
EPORSS as described in Algorithm~\ref{alg:EPORSS} employs a simple multi-objective evolutionary algorithm~\cite{friedrich2015maximizing,qian2019maximizing} to optimize the bi-objective problem. It starts from the empty set $\{0\}^n$ (line~1). In each iteration, a new solution $\bm{x'}$ is generated by applying the bit-wise mutation operator to an archived solution $\bm{x}$ randomly selected from the current population $P$ (lines~3--4); if $\bm{x'}$ is not dominated by any previously archived solution (line~5), it will be added into $P$, and meanwhile those solutions weakly dominated by $\bm{x'}$ will be removed (line~6). After running $T$ iterations, the solution with the largest $F$ value satisfying the size constraint in $P$ is selected (line~10).
\begin{algorithm}[t]
	\caption{EPORSS Algorithm}
	\label{alg:EPORSS}
	\textbf{Input}: all items $V=\{v_1,\ldots,v_n\}$, the objective function $F=\min_{1\le i\le m}f_i$, and a budget $k$\\
	\textbf{Parameter}: the number $T$ of iterations\\
	\textbf{Output}: a subset of $V$ with at most $k$ items\\
	\textbf{Process}:
	\begin{algorithmic}[1]
		\STATE Let $\bm{x}=\{0\}^{n}$, $P=\{\bm{x}\}$, and let $t=0$;
		\WHILE{ $t<T$}
		\STATE Select $\bm{x}$ from $P$ uniformly at random;
		\STATE Generate $\bm{x'}$ by flipping each bit of $\bm{x}$ with prob. $1/n$;
		\IF {$\nexists \bm{z} \in P$ such that $\bm{z} \succ \bm{x'}$}
		\STATE $P = (P \setminus \{\bm{z} \in P \mid \bm{x'} \succeq \bm z\}) \cup \{\bm{x'}\}$
		\ENDIF
		\STATE $t=t+1$
		\ENDWHILE
		\RETURN $\arg\max_{\bm x\in P, |\bm x| \leq k} F(\bm{x})$
	\end{algorithmic}
\end{algorithm}

We can see from the algorithm procedure that EPORSS naturally maintains a population of non-dominated solutions due to the bi-objective transformation. Meanwhile, the bit-wise mutation operator for reproduction is a global search operator, which flips each bit of a solution independently with probability $1/n$. These characteristics may make EPORSS have a better ability of escaping from local optima than the greedy algorithm. Our experiments will show that EPORSS can find better solutions by using more time. Here, we first show that though using an evolutionary process for optimization, EPORSS can also be theoretically grounded. Particularly, Theorem~\ref{thm:eporss} shows that EPORSS can achieve a similar approximation guarantee as the greedy algorithm. Note that $\expect{T}$ denotes the expected number $T$ of iterations, $\opgamma'=\min_{1\le i\le m}\min_{X:|X|=k-1} \opgamma_{X,k}(f_i)$, and $\beta'=\min_{X:|X|\le k-1}\beta_{X}$.
\begin{theorem}\label{thm:eporss}
	For robust subset selection in Definition~\ref{def:robust-subsetsel},  EPORSS with $\expect{T}\le 2ek^2n$ finds a subset $X\subseteq V$ with $|X|\leq k$ and
$F(X) \geq (1-e^{-\beta'\opgamma'})\cdot \opt$.
\end{theorem}
The proof can be accomplished by following the proof of Theorem 2 in~\cite{qian2019maximizing} and using Lemma~\ref{lem:onestep-F}. The main idea is 
analyzing the increase of a quantity $J_{\max}$, which is defined as the largest number $j\in \{0,1,\ldots,k\}$ such that in the population $P$, $ \exists \bmx$ with  $|\bmx|\le j$ and $F(\bmx)\ge \big(1-(1-\beta'\opgamma'/k)^j\big)\cdot\opt$. 
For the sake of completeness, we still provide the full proof in the appendix.

\section{Analysis of the Correlation Ratio}\label{sec:beta}

In the above two sections, we have shown that the approximation performance of the greedy algorithm and EPORSS can be theoretically bounded. However, the approximation guarantees depend on the submodularity ratio $\opgamma_{X,b}$ in Definition~\ref{def:subratio} and the correlation ratio $\beta_X$ in Definition~\ref{def:correlation-ratio}. Previous studies have shown that the lower bound on $\opgamma_{X,b}$ can be derived for several non-submodular objective functions~\cite{bian2017guarantees,elenberg2016restricted}.  In this section, we further derive a lower bound on $\beta_X$ for the application of robust influence maximization, showing the applicability of the derived approximation guarantees.

Influence maximization is to identify a set of influential users in social networks. Let a directed graph $G(V,E)$ represent a social network, where each node represents a user and each edge $(u,v)$ has a probability $p_{u,v}$ representing the strength of influence from user $u$ to user $v$. 
 As presented in Definition~\ref{def:IM}, the goal is to find a subset $X$ with size at most $k$ such that the expected number (denoted as $\sigma(X)$) of nodes activated by propagating from $X$ is maximized. 
\begin{definition}[Influence Maximization]\label{def:IM}
	Given a directed graph $G(V,E)$, edge probabilities $p_{u,v}$ where $(u,v) \in E$ is an edge, and a budget $k$, to find
	\begin{equation}
		\begin{aligned}
			\mathop{\arg\max}\nolimits_{X \subseteq V}\  \sigma(X) \quad \text{s.t.}\quad |X|\leq k.
		\end{aligned}
	\end{equation}
\end{definition}
Independence Cascade (IC) is a fundamental propagation model~\cite{kempe2003maximizing} to estimate the influence function $\sigma$, and will also be used in our analysis. Starting from a seed set $X$, it uses a set $A_t$ to record the nodes activated at time $t$, and at time $t+1$, each inactive neighbour $v$ of $u \in A_t$ becomes active with probability $p_{u,v}$. 
This process is repeated until no nodes get activated at some time.
Under the IC model, the objective function $\sigma$ is monotone submodular~\cite{kempe2003maximizing}.

For robust influence maximization, we are given several probability vectors $\bmth^1,\bmth^2,\ldots,\bmth^m$, where each $\bmth$ corresponds to a set of edge probabilities, and thus determines a specific influence function, denoted as $\sigma_{\bmth}$.
The goal is to maximize the worst of the $m$ influence functions $\sigma_{\bmth^1},\sigma_{\bmth^2},\ldots,\sigma_{\bmth^m}$, as shown in the following definition.
\begin{definition}[Robust Influence Maximization~\cite{he-kdd16-robust-inf,kalimeris-imcl19-robust-inf}]\label{def:robust-infmax}
	Given a directed graph $G(V,E)$, $m$ probability vectors $\bmth^1,\bmth^2,\ldots,\bmth^m$, and a budget $k$, to find
	\begin{equation}
		\begin{aligned}
			\mathop{\arg\max}\nolimits_{X \subseteq V}\  \min\nolimits_{1\le i\le m}\sigma_{\bmth^i}(X) \quad \text{s.t.}\quad |X|\leq k.
		\end{aligned}
	\end{equation}
\end{definition}
Theorem~\ref{thm:beta-infmax} gives a lower bound on $\beta_{X}$ for this application. The proof relies on Lemma~\ref{lem:beta-infmax}, which intuitively means that when two probability vectors are close, the difference of the influence spread determined by these two vectors can also be small. 
Lemma~\ref{lem:beta-infmax} is inspired from~\cite{chen-kdd16-robust-infmax}, but slightly refined. The detailed proof is provided in the appendix
due to space limitation, and we introduce the main proof idea here. Under the IC model, the influence spread $\sigma_{\bmth}(X)$ can be calculated as 
$\sum_{S\subseteq G}\pi_{\bmth}(S)\cdot \sigma_S(X)$~\cite{kempe2003maximizing},
where $\pi_{\bmth}(S)$ denotes the probability of sampling a subgraph $S$ from $G$ according to the edge probability vector $\bmth$, that is, each edge $(u,v)\in E$ appears in $S$ with probability $p_{u,v}$; and  $\sigma_S(X)$ denotes the number of nodes that can be reachable from $X$ on the given subgraph $S$, which is deterministic when $S$ is given. When two probability vectors only have little difference, the probabilities of sampling a specific subgraph are also slightly different, leading to the similar influence spread.
\begin{lemma}\label{lem:beta-infmax}
	For any two probability vectors $\bmth$ and $\bmth'$,  let $\delta(\bmth,\bmth')=\sum_i |\theta_i-\theta'_i|$ denote the difference between the two vectors, where $\theta_i$ and $\theta'_i$ denote the $i$-th element of $\bmth$ and $\bmth'$, respectively, and $|\cdot |$ denotes the absolute value of a real number. Then, for any $X\subseteq V$, we have $|\sigma_{\bmth}(X)-\sigma_{\bmth'}(X)|\le n\cdot\delta(\bmth,\bmth')$, where $n$ denotes the size of $V$.
\end{lemma}
Theorem 3 shows that $\beta_X$ can be lower bounded by $1-2en\cdot \delta_{\max}$.
 This is intuitive, because $\delta_{\max}$ measures the difference among the probability vectors, and thus the corresponding objective functions, while $\beta_X$ characterizes the similarity among the objective functions. Thus, a smaller $\delta_{\max}$ will lead to a larger lower bound on $\beta_X$.
 The proof is accomplished by analyzing 
 the ratio of marginal gains on $\sigma_{\bmth^i}$ by respectively adding $v^1$ and $v^i$ to $X$, where $v^1$ and $v^i$ are the items whose inclusion in $X$ can make $\sigma_{\bmth^1}$ and $\sigma_{\bmth^i}$ achieve the largest improvement, respectively.
 The detailed proof is provided in the appendix due to space limitation.
\begin{theorem}\label{thm:beta-infmax}
	For robust influence maximization in Definition~\ref{def:robust-infmax},  let $\delta_{\max}=\max_{1\le i,j\le m,i\neq j}\delta(\bmth^i,\bmth^j)$ denote the maximum difference among the $m$ probability vectors. Then, for any 
	$X\in \{X_0,X_1,\ldots,X_{k-1}\}$, 
	we have
 $\beta_X\ge 1-2en\cdot\delta_{\max}$,
 where $X_j$ ($0\le j\le k-1$) denotes the subset generated in the $j$-th iteration (note that $X_0=\emptyset$) of Algorithm~\ref{alg:Greedy}.
\end{theorem}

\section{Experiments}

In this section, we empirically compare the performance of the greedy algorithm, EPORSS and two previous algorithms, modified greedy~\cite{houtac21} and SATURATE~\cite{krause-jmlr08-rsos},
on the application of robust influence maximization. 
As suggested in~\cite{krause-jmlr08-rsos}, we set the parameter $\alpha$ of SATURATE to 1 such that the output subset will not violate the cardinality constraint. 
The number of iterations of EPORSS is set to $\lfloor 2ek^2n\rfloor$ as suggested by Theorem~\ref{thm:eporss}. As EPORSS is randomized, we repeat its run for 10 times independently, and report the average results and standard deviation.

The experiments are performed on two real-world data sets, \textit{ego-Facebook} and \textit{as-733}, downloaded from {\url{https://snap.stanford.edu/data/index.html}}.
The \textit{ego-Facebook} data set contains one network which consists of friend links from Facebook, and we use the perturbation interval method~\cite{he-kdd16-robust-inf} to sample multiple networks with different edge probabilities (which give rise to different objective functions). That is, we first estimate the propagation probability of one edge from node $u$ to $v$ by $p_{u,v}=\frac{weight(u,v)}{indegree(v)}$, as widely used in~\cite{chen2009efficient,goyal2011simpath}; and then sample from the interval  $[0.9\cdot p_{u,v},1.1\cdot p_{u,v}]$ multiple times to obtain different probabilities. 
For estimating the influence spread, we use the IC model as presented in Section~\ref{sec:beta}, which results in monotone submodular objective functions.

The \textit{as-733} data set contains 733 communication networks spanning an interval of 785 days, where the nodes and edges may be added or deleted over time. 
Thus, the two data sets reflect different causes for uncertainty: the edge probabilities are subject to perturbations, and the influence functions are learned from the networks with different structures. 
We use the general IC model~\cite{kempe2003maximizing} to estimate the influence spread on \textit{as-733}. For the general IC model, the probability of activating $v$ by $u$ is $p_v(u,S)$ instead of $p_{u,v}$, where $S$ is the set of neighbours that have already tried (and failed) to activate $v$. 
Specifically, for each network in \textit{as-733}, we set the probability $p_v(u,S)$ to $\min\{0.1+0.05\cdot |S|,1\}$, i.e., the probability of activating $v$ is 0.1 for the first try, and then the probability increases by 0.05 once a try fails.
Note that under the general IC model, the objective functions are generally non-submodular.

\begin{figure}[t!]\centering
	\hspace{1.3em}
	\begin{minipage}[c]{0.9\linewidth}\centering
		\includegraphics[width=1\linewidth]{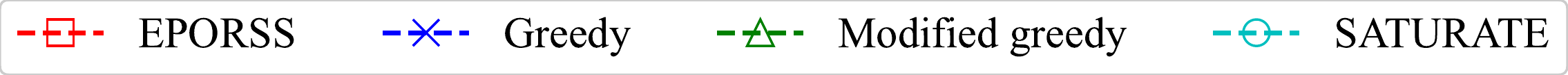}
	\end{minipage}\\\vspace{0.3em}
	\begin{minipage}[c]{0.48\linewidth}\centering
		\includegraphics[width=1\linewidth]{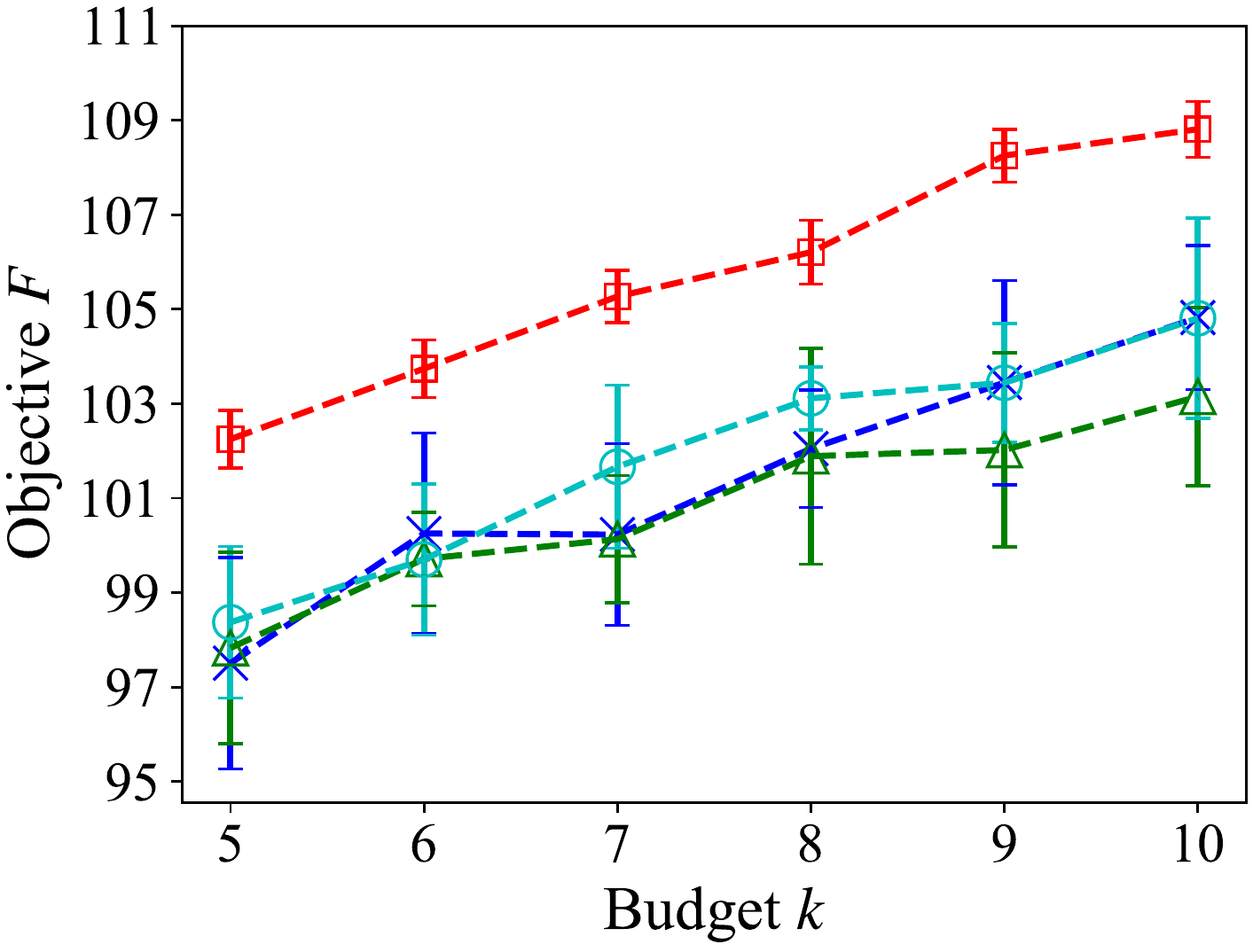}
	\end{minipage}\ \
	\begin{minipage}[c]{0.48\linewidth}\centering
		\includegraphics[width=1\linewidth]{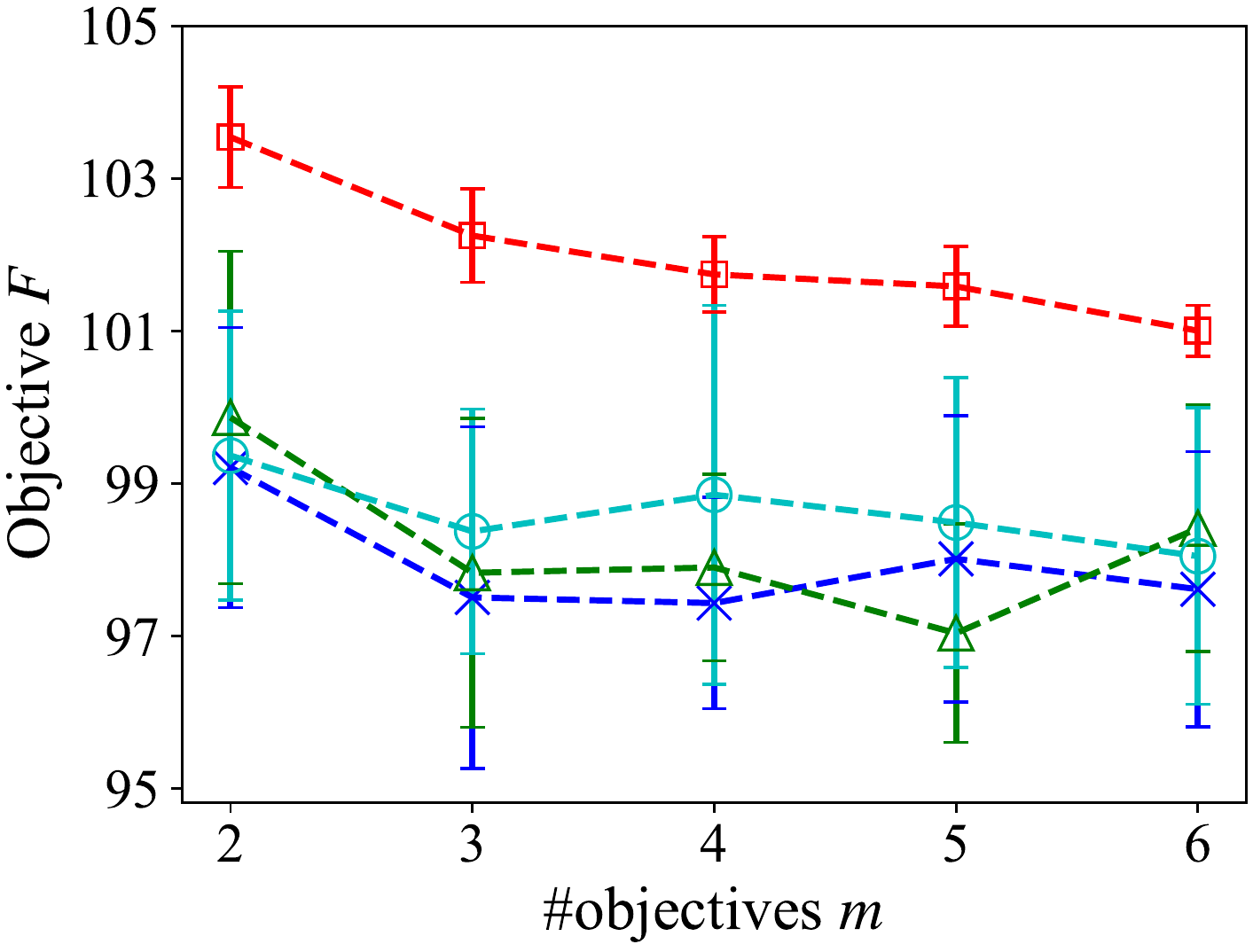}
	\end{minipage}\\\vspace{0.3em}
	\begin{minipage}[c]{1\linewidth}\centering
		\small(a) \textit{ego-Facebook}
	\end{minipage}\\\vspace{0.5em}
	\begin{minipage}[c]{0.48\linewidth}\centering
		\includegraphics[width=1\linewidth]{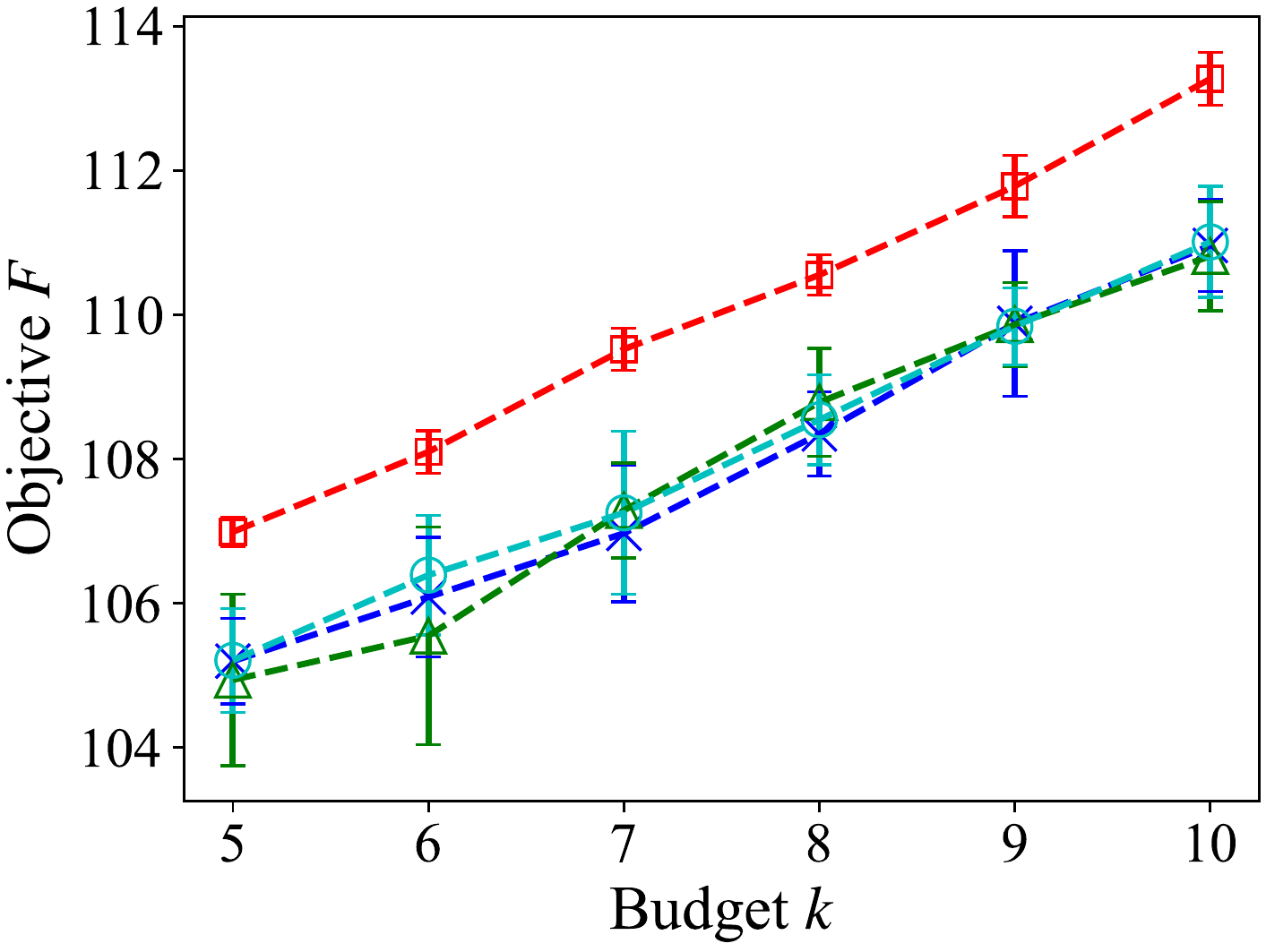}
	\end{minipage}\ \
	\begin{minipage}[c]{0.48\linewidth}\centering
		\includegraphics[width=1\linewidth]{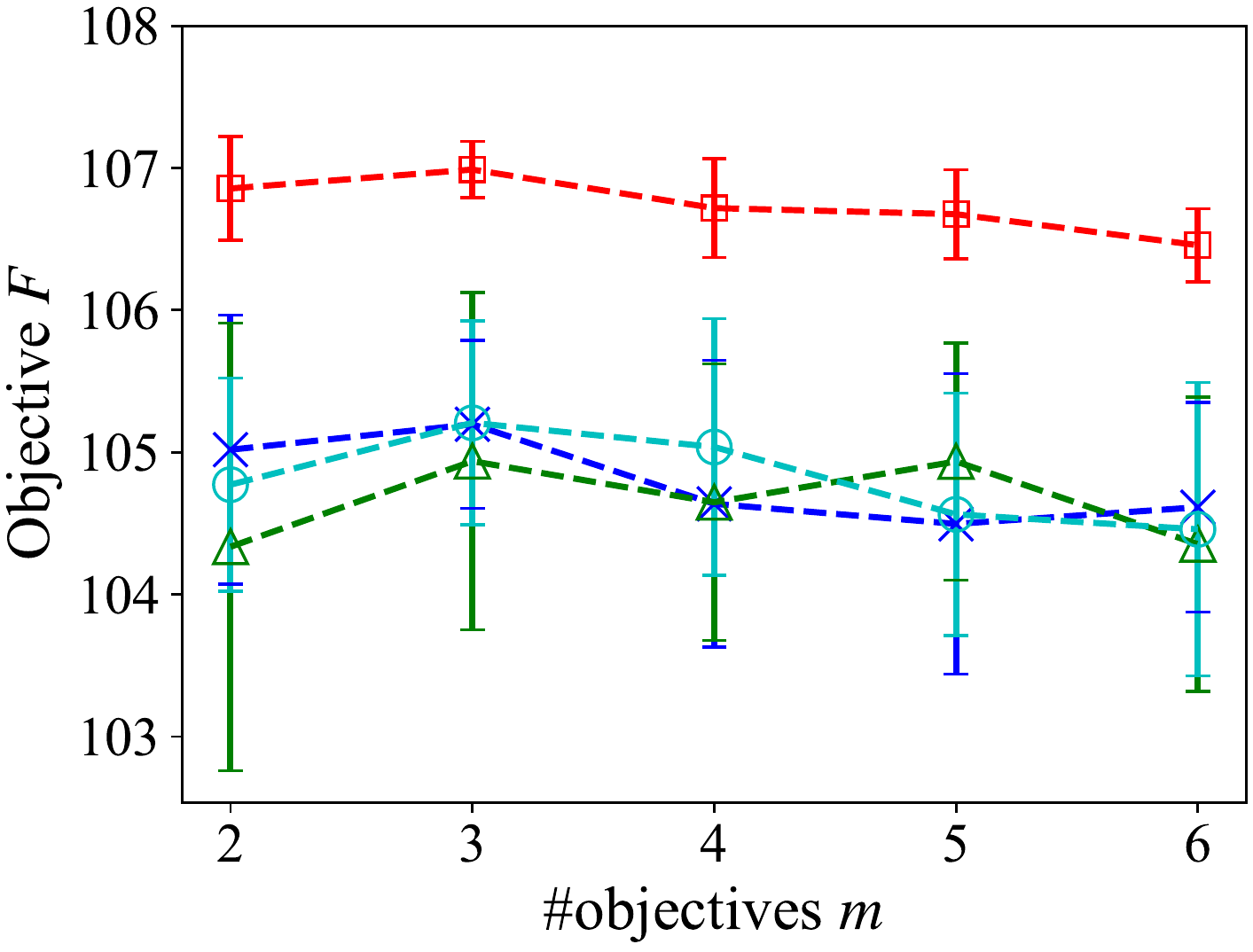}
	\end{minipage}\ \
	\small(b) \textit{as-733}
	\caption{Robust influence maximization on \textit{ego-Facebook} and \textit{as-733}. The objective $F$: the worst-case influence spread (the larger, the better). The left subfigure: objective $F$ vs. budget $k$; the right subfigure: objective $F$ vs. \#objectives  $m$.}\label{fig:std}
\end{figure}

For both data sets, we examine the performance of the compared algorithms when 
the budget $k$ changes from 5 to 10 (where the number $m$ of objective functions is fixed to 3), 
and also consider the scenario where the number $m$ of objective functions changes from 2 to 6 (where the budget $k$ is fixed to 5).
Note that we sample the most active 200 nodes from these two data sets for examination, as in~\cite{he-kdd16-robust-inf}. 
To estimate the influence spread $\sigma(\!X\!)$ of a subset $X$ of nodes,  we simulate the diffusion process 100 times independently and use the average as an estimation.
As the objective evaluation is noisy, we also repeat the running of the three deterministic algorithms (i.e., the greedy algorithm, the modified greedy algorithm and SATURATE)
10 times independently and report the average results.
We can observe from Figure~\ref{fig:std} that 
the three deterministic algorithms  
achieve the similar performance, while EPORSS is much better.

\textbf{Running time comparison.} 
The greedy algorithm in algorithm~\ref{alg:Greedy} needs to perform $(n-j+1)$ number of worst-case objective function evaluations in the $j$-th iteration.
Thus, the total running time, i.e., the number of  worst-case objective function evaluations, is $\sum_{j=1}^{k}(n-j+1)=(n-k/2+1/2)k$. 
The modified greedy algorithm starts from an empty set, and iteratively adds one item $v$ into the current subset $X$, such that all the objective functions can be improved as much as possible.
However, for each of the $m$ objective functions, the improvement  is measured by the ratio of the marginal gain achieved by adding $v$ and the largest possible marginal gain, which needs to be computed in the first place and thus requires extra $(n-j+1)$ number of objective function evaluations
in the $j$-th iteration. The process terminates after $k$ items are selected. Thus, the total running time is $\sum_{j=1}^{k}2(n-j+1)=(2n-k+1)k$, i.e., double that of the greedy algorithm. Though the extra objective function evaluations can be avoided by storing the marginal gain achieved by adding each item, extra storage cost is required. 
SATURATE employs a binary search framework, and each shrink over the search interval needs to perform a greedy subroutine, implying that the running time of SATURATE is substantially larger than that of the greedy algorithm. Thus, the greedy algorithm can achieve competitive performance to the modified greedy algorithm and SATURATE using less time.

EPORSS can achieve the best performance using more time, i.e., $2ek^2n$ iterations. However, $2ek^2n$ is only a theoretical upper bound, which may be too pessimistic in practice. By selecting the three deterministic algorithms (i.e., the greedy algorithm, the modified greedy algorithm and SATURATE) as the baselines, we plot the curve of the objective $F$ over the running time of EPORSS on  the two data sets, as shown in Figure~\ref{fig:runtime}. The $x$-axis is in $kn$, the running time order of the greedy algorithm. It can be observed that EPORSS takes only about $0.9kn$ and $0.6kn$ iterations to achieve the best performance on the two data sets, respectively, implying that EPORSS can be more efficient in practice.

\begin{figure}[t!]\centering
	\begin{minipage}[c]{0.48\linewidth}\centering
		\includegraphics[width=1\linewidth]{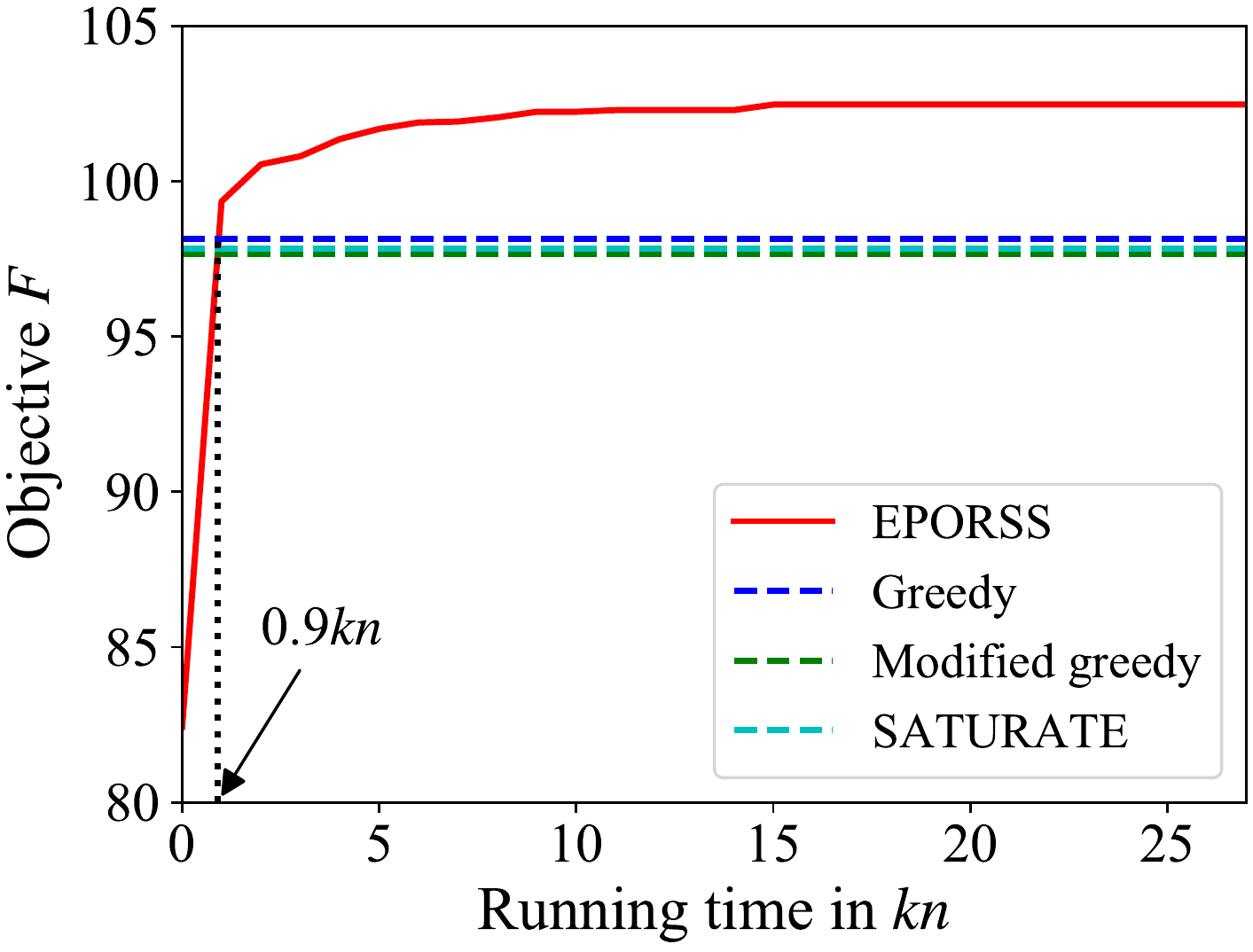}
	\end{minipage}\hspace{0.5em}
	\begin{minipage}[c]{0.48\linewidth}\centering
		\includegraphics[width=1\linewidth]{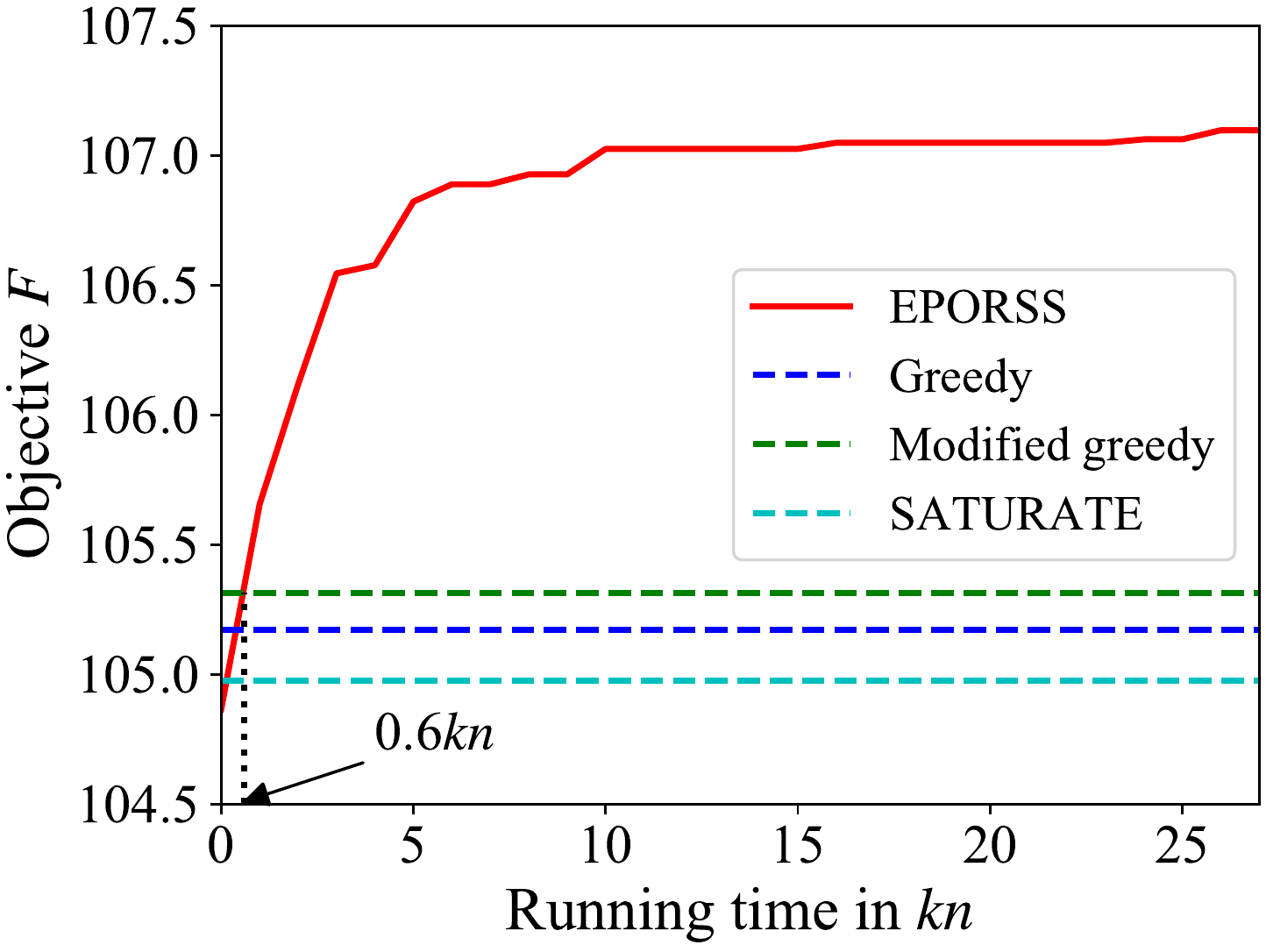}
	\end{minipage}\\\vspace{0.3em}
	\begin{minipage}[c]{0.45\linewidth}\centering
		\small(a) \textit{ego-Facebook}
	\end{minipage}\hspace{1em}
	\begin{minipage}[c]{0.45\linewidth}\centering
		\small(b) \textit{as-733}
	\end{minipage}
	\caption{The objective $F$ vs. the running time on robust influence maximization with $m=3$ and $k=5$. The objective $F$: the worst-case influence spread (the larger, the better). }\label{fig:runtime}
\end{figure}

\section{Conclusion}
In this paper, we study the robust subset selection problem with monotone
objective functions, and give two algorithms, i.e., the greedy algorithm and the evolutionary Pareto optimization algorithm EPORSS, with bounded approximation guarantees. The experimental results on the application of robust influence maximization show that the greedy algorithm can achieve competitive performance to the previous algorithms, but is more efficient; while EPORSS can achieve better performance by using more time. Thus, these two algorithms can help solve various real-world robust subset selection problems better. 
Note that we have considered the worst-case running time complexity of EPORSS in theoretical analysis, and the empirical results show that EPORSS can be efficient in practice. Thus, it would be interesting to analyze the average-case running time complexity, which may reflect the practical performance of EPORSS more accurately.


\small
\bibliographystyle{named}
\bibliography{eporss-bib}

\appendix

\newpage
   
\newpage
\normalsize
\section*{Appendix}
\vspace{1em}
\subsection*{Procedure of the Modified Greedy Algorithm}
\begin{algorithm}[h]
	\caption{Modified Greedy Algorithm}
	\textbf{Input}: all items $V=\{v_1,v_2,\ldots,v_n\}$, 
	$m$ objective functions $f_1,f_2,\ldots,f_m$, 
	and a budget $k$\\
	\textbf{Output}: a subset of $V$ with $k$ items\\
	\textbf{Process}:
	\begin{algorithmic}[1]
		\STATE Let $j=0$ and $X_j=\emptyset$;
		\WHILE {$j<k$}
		\FOR {$i=1$ to $m$}
		\STATE Let $a_{i}^*=\arg\max_{a\in V \setminus X_j}\big(f_i(X_j\cup \{a\}-f_i(X_j)\big)$
		\ENDFOR
		\STATE  Let $v^*=\arg\max_{v \in V \setminus X_j} \min_i \frac{f_i(X_j \cup \{v\})-f_i(X_j)}{f_i(X_j \cup \{a_i^*\})-f_i(X_j)}$;
		\STATE  Let $X_{j+1}=X_{j} \cup \{v^*\}$, and $j=j+1$
		\ENDWHILE
		\RETURN {$X_k$}
	\end{algorithmic}
\end{algorithm}

\subsection*{Omitted Proofs}
\vspace{0.5em}
\begin{proof}[Proof of Lemma~\ref{lem:onestep-F}]
	Let 
	\begin{equation}
		\hat{v}\in \mathop{\arg\max}\nolimits_{v\in V\setminus X}\min\nolimits_{1\le i\le m}\frac{f_i(X\cup \{v\})-f_i(X)}{f_i(X\cup \{v^i\})-f_i(X)},
	\end{equation}
	where $v^i\in \arg\max_{v\in V\backslash X} f_i(X\cup \{v\})$. 
	Then, $\forall 1\le i\le m$,
	\begin{equation}
		\begin{aligned}
			&f_i(X\cup\{\hat{v}\})\!-\! f_i(X)\\
			&= \frac{f_i(X\cup \{\hat{v}\})-f_i(X)}{f_i(X\cup \{v^i\})-f_i(X)}\cdot \big( f_i(X\cup \{v^i\})-f_i(X)\big)\\
			&\ge \beta_{X} \cdot \big( f_i(X\cup \{v^i\})-f_i(X)\big)\\
			&\ge \beta_X\cdot (\opgamma_{X,k}^{\min}/k)\cdot (\opt_i-f_i(X))\\
			&\ge  (\beta_X\opgamma_{X,k}^{\min}/k)\cdot (\opt-f_i(X)),
		\end{aligned}
	\end{equation}
	where the first inequality is by the definitions of $\hat{v}$ and $\beta_X$, the second inequality is by Lemma~\ref{lem:onestep-f_i} and the definition of $\opgamma_{X,k}^{\min}$, and the last inequality holds because  $\opt=F(X^*)=\min_{1\le j\le m}f_j(X^*)\le  f_i(X^*)\le \opt_i$, where $X^*\in \mathop{\arg\max}_{X\subseteq V,|X|\le k}F(X)$.
	Thus, 
	\begin{align}
		&f_i(X\cup \{\hat{v}\})-(\beta_X\opgamma_{X,k}^{\min}/k)\cdot \opt \\
		&\ge  (1-\beta_X\opgamma_{X,k}^{\min}/k)f_i(X)
		\ge (1-\beta_X\opgamma_{X,k}^{\min}/k)F(X),
	\end{align}
	which holds for any $1 \leq i \leq m$, implying
	\begin{equation}
		F(X\cup \{\hat{v}\})-F(X) \ge (\beta_X\opgamma_{X,k}^{\min}/k)\cdot (\opt-F(X)).
	\end{equation}
	Thus, the lemma holds.
\end{proof}

\vspace{0.5em}
\begin{proof}[Proof of Theorem~\ref{thm:eporss}]
	The theorem is proved by analyzing the increase of a quantity $J_{\max}$, which is defined as \begin{align}
		J_{\max}=\max \big\{& j\in \{0,1,\ldots,k\}\mid \exists \bmx\in P \ \text{s.t.}\  |\bmx|\le j \wedge \\
		&F(\bmx)\ge \big(1-(1-\beta'\opgamma'/k)^j\big)\cdot\opt\big\}.
	\end{align}
	We only need to analyze the expected number of iterations until $J_{\max}$  increases to $k$,  implying $\exists \bmx\in P$ such that $|\bmx|\le k$ and $F(\bmx)\ge (1-(1-\frac{\beta'\opgamma'}{k})^k)\cdot \opt \ge (1-e^{-\beta'\opgamma'})\cdot \opt $. 
	
	The initial value of $J_{\max}$ is $0$, because EPORSS starts from $\{0\}^n$. 
	Assume that currently $J_{\max}=i<k$, and let $\bm{x}$ be the corresponding solution with the value $i$, i.e., $|\bmx| \leq i$ and
	\begin{equation}\label{eq:inductive}
		F(\bmx)\ge \big(1-(1-\beta'\opgamma'/k)^i\big) \cdot\opt.
	\end{equation}
	We first show that $J_{\max}$ cannot decrease. If $\bm{x}$ is kept in $P$, $J_{\max}$ obviously will not decrease. If $\bm{x}$ is deleted from $P$ in line~6 of Algorithm~\ref{alg:EPORSS}, the newly included solution $\bm{x}'$ must weakly dominate $\bm{x}$, i.e., $F(\bm{x}') \geq F(\bm{x})$ and $|\bmx'| \leq |\bmx|$, implying $J_{\max} \geq i$.
	
	We next show that $J_{\max}$ can increase by at least 1 in each iteration with probability at least $1/(enP_{\max})$, where $P_{\max}$ denotes the largest size of the population $P$ during the optimization procedure of EPORSS. By Lemma~\ref{lem:onestep-F}, flipping one specific 0-bit of $\bm{x}$ (i.e., adding a specific item into $\bm{x}$) can generate a new solution $\bm{x}'$ such that 
	\begin{equation}\label{eq:J-increase}
		\begin{aligned}
			F(\bm{x}')-F(\bm{x}) 
			&\ge (\beta_{\bmx} \opgamma^{\min}_{\bmx,k}/k)\cdot (\opt-F(\bmx))\\
			&\ge (\beta' \opgamma'/k)\cdot (\opt-F(\bmx)),
		\end{aligned}
	\end{equation}
	where the last inequality holds because $\beta'\!\!=\!\min_{\!X:|X|\le k-1}\!\beta_{X}$ $ \leq \beta_{\bmx}$, and $\opgamma'=\min_{1\le i\le m}\min_{X:|X|=k-1} \opgamma_{X,k}(f_i)  \leq \min_{1\le i\le m}\gamma_{\bmx,k}(f_i)=\gamma^{\min}_{\bmx,k}$  due to the fact that $\gamma_{X,k}(f)$ is  monotone non-increasing with respect to $X$. Note that $|\bm{x}|\leq i< k$.
	Combining Eqs.~\eqref{eq:inductive} and~\eqref{eq:J-increase} leads to
	\begin{equation}
		\begin{aligned}
			F\left(\bmx'\right) 
			&\ge (1-\beta'\gamma'/k) \cdot F(\bmx)+(\beta'\gamma'/k) \cdot \opt\\
			& =\big(1-(1-\beta'\gamma'/k)^{i+1}\big) \opt.
		\end{aligned}
	\end{equation}
	Note that $|\bmx'|=|\bmx|+1\le i+1$. Thus, $\bmx'$ will be included into $P$; otherwise, $\bm{x}'$ must be dominated by one solution in $P$ (line~5 of Algorithm~\ref{alg:EPORSS}), implying that $J_{\max}$ has already been larger than $i$, which contradicts with the assumption $J_{\max}=i$. After including $\bm{x}'$, $J_{\max} \geq i+1$. Thus, $J_{\max}$ can increase by at least 1 in one iteration with probability at least $(1/P_{\max})\cdot (1/n)(1-1/n)^{n-1} \geq 1/(enP_{\max})$, where $1/P_{\max}$ is a lower bound on the probability of selecting $\bm{x}$ in line~3 of Algorithm~\ref{alg:EPORSS} due to uniform selection, and $1/n(1-1/n)^{n-1}$ is the probability of flipping a specific bit of $\bm{x}$ while keeping the other bits unchanged in line~4. This implies that the expected number of iterations to increase $J_{\max}$ by at least 1 is at most $enP_{\max}$.
	
	During the optimization procedure of EPORSS, the solutions maintained in the population $P$ must be incomparable. Thus, for any $i\in \{0,1,\ldots,2k-1\}$, $P$ contains at most one solution with size $i$. Note that the solutions with size larger than $2k-1$ have value $-\infty$ on the first objective, and must be excluded from $P$. Then, we have $P_{\max}\le 2k$.
	Since $J_{\max}$ needs to be increased by at most $k$ times to reach $k$, the expected number  $\expect{T}$ of iterations for finding the desired approximation guarantee is at most $2ek^2n$.
\end{proof}

\vspace{0.5em}
\begin{proof}[Proof of Lemma~\ref{lem:beta-infmax}]
	As in the proof of Theorem 2.2 in~\cite{kempe2003maximizing}, the influence function $\sigma(X)$ can be calculated as
	\begin{equation}\label{eq:subgraph}
		\sigma_{\bmth}(X)=\sum_{S\subseteq G}\pi_{\bmth}(S)\cdot \sigma_S(X),
	\end{equation}
	where $\pi_{\bmth}(S)$ denotes the probability of sampling a subgraph $S$ from $G$ according to the edge probability vector $\bmth$, that is, each edge $(u,v)\in E$ appears in $S$ with probability $p_{u,v}$; and  $\sigma_S(X)$ denotes the number of nodes that can be reachable from $X$ on the given subgraph $S$. Note that $\sigma_S(X)$ is deterministic when $S$ is given. 
	
	Now we consider a probability vector $\bmth^{(1)}$, such that $\theta^{(1)}_1=\theta'_1$ and $\forall i>1:\theta_i^{(1)}=\theta_i$, i.e., $\bmth^{(1)}$ is the same as $\bmth$ except for the first element. Denote the edge corresponding to $\theta^{(1)}_1$ as $e_1$. By Eq.~\eqref{eq:subgraph}, we have
		\begin{align}
			&\sigma_{\bmth^{(1)}}(X)\\
			&=\!\sum_{S\subseteq G,e_1\in S}\pi_{\bmth^{(1)}}(S) \!\cdot\! \sigma_S(X)\!+\!\sum_{S\subseteq G,e_1\notin S}\pi_{\bmth^{(1)}}(S)\!\cdot\! \sigma_S(X)\\
			&=\sum_{S\subseteq G,e_1\in S}\pi_{\bmth^{(1)}}(S\mid e_1\in S)\prob(e_1\in S)\cdot \sigma_S(X)\\
			&\quad+\sum_{S\subseteq G,e_1\notin S}\pi_{\bmth^{(1)}}(S|e_1\notin S)\prob(e_1\notin S)\cdot \sigma_S(X)\\
			&=\ \theta^{(1)}_1\sum_{S\subseteq G,e_1\in S}\pi_{\bmth^{(1)}}(S|e_1\in S)\cdot \sigma_S(X)\\
			&\quad+(1-\theta^{(1)}_1)\sum_{S\subseteq G,e_1\notin S}\pi_{\bmth^{(1)}}(S|e_1\notin S)\cdot \sigma_S(X)\\
			&=\ \theta^{(1)}_1\sigma(X|e_1 \text{ is live})+(1-\theta^{(1)}_1)\sigma(X|e_1 \text{ is blocked}),
		\end{align}
	where the third equality holds because each edge of $S$ is chosen independently, and ``$e_1$ is live" (or ``$e_1$ is blocked") means that it is always successful (or unsuccessful) when a node tries to activate its neighbouring node through the edge $e_1$. Then, we have
	\begin{equation}
		\begin{aligned}\label{eq:sigma-diff}
			&\sigma_{\bmth^{(1)}}(X)-\sigma_{\bmth}(X)\\
			&=(\theta^{(1)}_1-\theta_{1})\sigma(X|e_1 \text{ is live})\\
			&\quad +(\theta_{1}-\theta^{(1)}_1)\sigma(X|e_1 \text{ is blocked})\\
			&=(\theta'_1-\theta_{1})\cdot (\sigma(X|e_1 \text{ is live})-\sigma(X|e_1 \text{ is blocked}))\\
			&\le |\theta'_1-\theta_1|\cdot (\sigma(X|e_i \text{ is live})-\sigma(X|e_i \text{ is blocked})) \\
			&\le |\theta'_1-\theta_1|\cdot n,
		\end{aligned}
	\end{equation}
   where the last inequality holds  because there are at most $n$ nodes in the graph.
	Next, we consider the sequence of probability vectors $\bmth^{(1)}, \bmth^{(2)},\ldots$,  where $\bmth^{(i)}$ is the same as $\bmth^{(i-1)}$ except for the $i$-th element (which is set to $\theta'_i$).  That is, $\theta^{(i)}_i=\theta'_i$ and $\forall j\neq i:\theta_j^{(i)}=\theta_j^{(i-1)}$.
	By Eq.~\eqref{eq:sigma-diff}, we have 
	\begin{align}
		&\sigma_{\bmth'}(X)-\sigma_{\bmth}(X)\\
		&=\sigma_{\bmth^{(1)}}(X)-\sigma_{\bmth}(X)+\sigma_{\bmth^{(2)}}(X)-\sigma_{\bmth^{(1)}}(X)+\ldots\\
		&\le \sum_{i}|\theta'_i-\theta_i|\cdot n
		=n\cdot \delta(\bmth,\bmth').
	\end{align}
	As the above equation also holds for $\sigma_{\bmth}(X)-\sigma_{\bmth'}(X)$, we conclude the lemma.
\end{proof}

\vspace{0.3em}
\begin{proof}[Proof of Theorem~\ref{thm:beta-infmax}]
	In the following proof, we use $\sigma_i$ to represent $\sigma_{\bmth^i}$  for short.
	Let $v^i\in \arg\max_{v\in V\backslash X} \sigma_i(X\cup \{v\})$. For $1\le i\le m$, we have 
	\begin{align}
		&\frac{\sigma_i(X\cup \{v^1\})-\sigma_i(X)}{\sigma_i(X\cup \{v^i\})-\sigma_i(X)}\\
		&=\frac{\sigma_i(X\cup \{v^1\})\!-\!\sigma_i(X\cup \{v^i\})\!+\!\sigma_i(X\cup \{v^i\})\!-\!\sigma_i(X)}{\sigma_i(X\cup \{v^i\})-\sigma_i(X)}\\
		&=1-\frac{\sigma_i(X\cup \{v^i\})-\sigma_i(X\cup \{v^1\})}{\sigma_i(X\cup \{v^i\})-\sigma_i(X)}\\
		&\ge 1-\frac{\sigma_1(X\!\cup\! \{v^i\})\!+\!n\!\cdot\!\delta_{\max}\!-\!(\sigma_1(X\!\cup\! \{v^1\})\!-\!n\!\cdot\!\delta_{\max})}{\sigma_i(X\!\cup \!\{v^i\})-\sigma_i(X)}\\
		&\ge 1-\frac{2n\cdot\delta_{\max}}{\sigma_i(X\cup \{v^i\})-\sigma_i(X)},
	\end{align}
	where the first inequality is by Lemma~\ref{lem:beta-infmax} and the definition of $\delta_{\max}$. To derive a lower bound on $\sigma_i(X\cup \{v^i\})-\sigma_i(X)$, we have
	\begin{equation}
		\begin{aligned}
			&\sigma_i(X\cup \{v^i\})-\sigma_i(X)\\
			&\ge \frac{1}{n-|X|}\cdot \sum_{v\in V\backslash X}(\sigma_i(X\cup\{v\})-\sigma_i(X))\\
			&\ge \frac{1}{n-|X|}\cdot (\sigma_i(V)-\sigma_i(X))\ge \frac{n-\sigma_i(X)}{n},
		\end{aligned}
	\end{equation}
	where the first inequality is by the definition of $v^i$, and the second inequality is by the submodularity of $\sigma_i$. In the following discussion, we can pessimistically assume that  $\sigma_i(X)< (1-1/e)n$; because otherwise, 
	the greedy algorithm has already generated a subset $X$ that performs well on $\sigma_i$, i.e., achieves the optimal $(1-1/e)$ approximation ratio for $\sigma_i$, implying that there is no need to consider $\sigma_i$.
	Then, we get 
	\begin{equation}
		\begin{aligned}
			\frac{\sigma_i(X\cup \{v^1\})-\sigma_i(X)}{\sigma_i(X\cup \{v^i\})-\sigma_i(X)}
			\ge 1-2en\cdot\delta_{\max}.
		\end{aligned}
	\end{equation}
	Note that the above equation holds for any $1\le i\le m$, implying 
	$\beta_X\ge 1-2en\cdot\delta_{\max}$. Thus, the theorem holds.
\end{proof}

\end{document}